\documentclass[letterpaper]{article} 
\usepackage[]{aaai2026}  
\makeatletter
\gdef\copyright@on{}
\makeatother
\usepackage{times}  
\usepackage{helvet}  
\usepackage{courier}  
\usepackage[hyphens]{url}  
\usepackage{graphicx} 
\usepackage{natbib}  
\usepackage{caption} 
\title{Parallel Lifted Planning via Semi-Naive Datalog Evaluation}

\author{
    Dominik Drexler\textsuperscript{\rm 1},
    Oliver Joergensen\textsuperscript{\rm 2},
    Jendrik Seipp\textsuperscript{\rm 3},
}
\affiliations{
    Link{\"o}ping University, Sweden\\
    \textsuperscript{\rm 1}dominik.drexler@liu.se,
    \textsuperscript{\rm 2}oliver.harold.joergensen@liu.se,
    \textsuperscript{\rm 3}jendrik.seipp@liu.se,
}

\usepackage{algorithm}
\usepackage[indLines=false,commentColor=black]{algpseudocodex}
\usepackage{booktabs}
\usepackage{multirow,makecell}
\usepackage{amsmath,amssymb,amsthm}
\usepackage{multicol}
\usepackage{stmaryrd}
\usepackage{mathtools}
\usepackage{subcaption}
\usepackage{tikz}
\usetikzlibrary{shapes,arrows,arrows.meta}
\usetikzlibrary{automata}
\usetikzlibrary{positioning}
\usetikzlibrary{backgrounds}
\usetikzlibrary{patterns,patterns.meta,angles}
\usepackage{xspace}

\usepackage{courier}
\usetikzlibrary{graphs, quotes}
\usepackage{forest}
\usepackage{pgfplots}
\pgfplotsset{compat=1.18}
\usepackage[mode=buildnew]{standalone}
\usepackage{paralist}
\usepackage[acronym]{glossaries}

\usepackage[main=american]{babel}
\useshorthands*{"}
\defineshorthand{"-}{\babelhyphen{hard}}
\defineshorthand{"=}{\babelhyphen{–}}
\defineshorthand{"/}{\babelhyphen{/}}
\usepackage{stmaryrd}

\newcommand{\Omit}[1]{}

\newcommand{\defined}[1]{\emph{#1}}

\definecolor{tblue}{RGB}{0,119,187}
\definecolor{tcyan}{RGB}{51,187,238}
\definecolor{tteal}{RGB}{0,153,136}
\definecolor{torange}{RGB}{238,119,51}
\definecolor{tred}{RGB}{204,51,17}
\definecolor{tmagenta}{RGB}{238,51,119}
\definecolor{tgray}{RGB}{187,187,187}

\newtheorem{corollary}{Corollary}
\newtheorem{lemma}{Lemma}
\newtheorem{theorem}{Theorem}

\theoremstyle{definition}

\newtheorem{definition}{Definition}

\definecolor{darkgreen}{rgb}{0.0, 0.2, 0.13}
\iftrue
    
    \newcommand{\oliver}[1]{\textcolor{orange}{Oliver: #1}}
    \newcommand{\dominik}[1]{\textcolor{magenta}{Dominik: #1}}
    \newcommand{\jendrik}[1]{\textcolor{blue}{Jendrik: #1}}
    
\else
    
    \newcommand{\oliver}[1]{}
    \newcommand{\dominik}[1]{}
    \newcommand{\jendrik}[1]{}
    
\fi
\newcommand{\notes}[1]{}

\newcommand{\inlinecite}[1]{\citet{#1}}
\newcommand{\egcite}[1]{\citep[e.g.,][]{#1}}

\newcommand{\deltakpkc}{$\Delta$-KPKC}
\newcommand{\vdg}{\ensuremath{\mathcal{G}}}

\newcommand{\tup}[1]{\ensuremath{\langle #1 \rangle}}

\newcommand{\naturals}{\ensuremath{\mathbb{N}}}

\newcommand{\predicates}{\ensuremath{\mathcal{P}}}
\newcommand{\predicate}{\ensuremath{P}}
\newcommand{\atom}{\ensuremath{p}}
\newcommand{\groundatom}{\ensuremath{\overline{p}}}
\newcommand{\groundatoms}{\ensuremath{\overline{\predicate}}}
\newcommand{\literal}{\ensuremath{\ell}}
\newcommand{\groundliteral}{\ensuremath{\overline{\literal}}}

\newcommand{\actionschemas}{\ensuremath{\mathcal{A}}}
\newcommand{\actionschema}{\ensuremath{a}}
\newcommand{\groundaction}{\ensuremath{\overline{a}}}
\newcommand{\objects}{\ensuremath{\mathcal{O}}}
\newcommand{\object}{\ensuremath{o}}
\newcommand{\variables}{\ensuremath{\mathcal{X}}}
\newcommand{\variable}{\ensuremath{x}}
\newcommand{\initial}{\ensuremath{\state_0}}
\newcommand{\goal}{\ensuremath{\mathcal{G}}}
\newcommand{\arity}[1]{\ensuremath{\mathit{ar}(#1)}}
\newcommand{\element}{\ensuremath{y}}
\newcommand{\groundelement}{\ensuremath{\overline{y}}}

\newcommand{\substitution}{\ensuremath{\rho}}
\newcommand{\substitutions}{\ensuremath{\mathcal{R}}}
\newcommand{\ground}{\ensuremath{\text{gr}}}

\newcommand{\precondition}[1]{\ensuremath{\mathit{pre}(#1)}}
\newcommand{\effect}[1]{\ensuremath{\mathit{eff}(#1)}}
\newcommand{\pprecondition}[1]{\ensuremath{\mathit{pre}^+(#1)}}
\newcommand{\msprecondition}[1]{\ensuremath{\mathit{pre}_s^-(#1)}}
\newcommand{\mprecondition}[1]{\ensuremath{\mathit{pre}^-(#1)}}
\newcommand{\peffect}[1]{\ensuremath{\mathit{eff}^+(#1)}}
\newcommand{\meffect}[1]{\ensuremath{\mathit{eff}^-(#1)}}

\newcommand{\state}{\ensuremath{s}}

\newcommand{\consistencygraph}{\ensuremath{\mathcal{G}}}

\newcommand{\vertices}{\ensuremath{\mathcal{V}}}
\newcommand{\vertex}{\ensuremath{v}}

\newcommand{\edges}{\ensuremath{\mathcal{E}}}
\newcommand{\removaledges}{\ensuremath{\mathcal{I}}}
\newcommand{\clique}{\ensuremath{\mathcal{C}}}

\newcommand{\rank}{\ensuremath{\mathit{rank}}}

\newcommand{\dprogram}{\ensuremath{\mathcal{D}}}
\newcommand{\dfacts}{\ensuremath{\mathcal{F}}}
\newcommand{\dconsequence}{\ensuremath{\mathcal{T}}}

\newcommand{\ddelta}{\ensuremath{\Delta}}
\newcommand{\ddatabase}{\ensuremath{J}}
\newcommand{\drules}{\ensuremath{\mathcal{R}}}
\newcommand{\drule}{\ensuremath{r}}
\newcommand{\dgroundrule}{\ensuremath{\overline{r}}}

\newcommand{\dworker}{\ensuremath{w}}
\newcommand{\dmodel}{\ensuremath{\mathcal{M}}}
\newcommand{\dbody}{\ensuremath{\mathrm{B}}}
\newcommand{\dhead}{\ensuremath{\mathrm{H}}}
\newcommand{\diter}{\ensuremath{t}}

\algrenewcommand\algorithmicindent{1.0em}%

\usepackage{amsmath}
\usepackage{booktabs}
\usepackage{nicefrac}

\usepackage{tikz}
\usepackage{pgfplots}
\usepgfplotslibrary{statistics}
\pgfplotsset{compat=1.17}
\usetikzlibrary{patterns}
\usetikzlibrary{arrows}
\usetikzlibrary{arrows.meta}
\usetikzlibrary{graphs,graphs.standard,calc}
\usetikzlibrary{shapes,snakes,decorations}
\usetikzlibrary{positioning}
\usetikzlibrary{matrix}
\usetikzlibrary{fit}

\usepackage{graphicx}
\usepackage{multirow}
\usepackage{microtype}
\usepackage{tangocolors}

\begin{document}

\maketitle

\begin{abstract}

Lifted classical planners operate directly on first-order planning tasks to avoid the computationally demanding grounding step. However, lifted planning is typically slower, as planners must repeatedly instantiate ground structures during search. Many core components of lifted classical planning, such as successor generation, axiom evaluation, task grounding, and delete-relaxed heuristics, have previously been studied through the lens of Datalog evaluation. We build upon this line of work and extend it by developing and analyzing an execution model with two levels of parallelism: rule-level parallelism and grounding parallelism. We further specialize this solver for planning-specific workloads with a grounder based on clique enumeration, which we extend to support semi-naive Datalog evaluation.
Our experimental evaluation using greedy best-first search with the FF heuristic shows that our implementation already solves more tasks than the baselines on a single core, and the gap widens as additional cores are used. Moreover, on hard-to-ground tasks where on average 97.6\% of the total runtime is spent in Datalog execution, the proposed execution model exhibits an average parallel fraction of 92.4\%, while achieving up to a 6-fold speedup on 8 cores in practice.

\end{abstract}

\section{Introduction}

Classical planning systems commonly solve problems using a grounded representation of the planning task. While grounding enables efficient search and heuristic computation, it may become infeasible when the number of ground atoms and actions exceeds available memory. Lifted planners avoid this issue by operating directly on the first-order representation of the planning task. However, lifted reasoning introduces additional overhead, as planners repeatedly generate and evaluate ground structures during search.

Several components of lifted classical planning correspond to rule evaluation problems. In particular, successor generation \cite{correa-et-al-icaps2020}, axiom evaluation, task grounding \cite{helmert-aij2009}, and delete-relaxed planning graph heuristics \cite{correa-et-al-icaps2021,correa-et-al-aaai2022} can be expressed as Datalog rule evaluations over a database of facts. Previous work exploits these connections and shows that semi-naive Datalog evaluation \cite{bancilhon-ramakrishnan-icmd-1986,ceri-et-al-ieee1989} can be used within lifted planners to address difficult-to-ground tasks. However, these approaches evaluate rules sequentially and do not exploit inherent parallelism.

In this paper, we propose a \emph{parallel lifted planner} based on the connection between lifted planning and Datalog. Numerous parallel Datalog architectures have been proposed, with scalability varying considerably across techniques and benchmarks \egcite{jordan-et-al-iccav2016,shkapsky-et-al-sigmod2016}, making it challenging to identify designs that perform well for planning workloads, where general-purpose engines are often not directly applicable. Parallel Datalog systems can broadly be categorized into synchronous and asynchronous architectures: in synchronous evaluation, rule applications proceed in globally coordinated rounds with barriers, whereas asynchronous evaluation allows independent updates to shared relations. While synchronous architectures are less expressive, they are not necessarily less efficient in practice and can even outperform asynchronous ones due to more controlled communication patterns \cite{das-zaniolo-tplp2019}. We therefore adopt a synchronous semi-naive execution model with two complementary levels of parallelism: rule-level parallelism evaluates rules of a stratum concurrently, while grounding parallelism distributes the enumeration of rule instances across workers.

Our Datalog solver is specialized for classical planning using a grounder based on $k$-clique enumeration in $k$-partite substitution consistency graphs \cite{stahlberg-ecai2023}. In contrast to join-based database techniques \cite{abiteboul-et-al-1995}, this sidesteps the worst-case blowup associated with intermediate joins \cite{ngo-et-al-jacm2018} and makes grounding largely output-sensitive, up to polynomial bookkeeping overhead. Soundness and completeness hold for rules whose body literals have arity at most two; the remaining cases are handled by an explicit applicability test. We extend this approach to semi-naive evaluation by seeding clique enumeration at newly consistent edges \cite{wang-et-al-acmmd2024}. We evaluate on classical planning benchmarks using greedy best-first search with the FF heuristic \cite{hoffmann-nebel-jair2001}, achieving speedups of 2--4x and up to 6x on 8 cores.

The main contributions of this work are:

\begin{itemize}
    \item We propose \deltakpkc{}, a semi-naive Datalog solver for lifted classical planning based on $k$-clique enumeration in $k$-partite graphs.
    \item We develop a synchronous parallel execution architecture for \deltakpkc{} with two levels of parallelism: rule-level parallelism and grounding parallelism.
    \item We provide a detailed empirical analysis of the effectiveness and scalability of \deltakpkc{} on classical planning benchmarks, including average empirical speedup limits.
\end{itemize}

\section{Related Work}

\subsection{The Lifted FF Heuristic}

The FF heuristic \cite{hoffmann-nebel-jair2001} is one of the most widely used heuristics in classical planning. It explores the delete-relaxed planning graph over grounded actions, followed by backchaining from the goal, to retrieve informative heuristic estimates and preferred actions. \inlinecite{correa-et-al-aaai2022} investigated lifted variants of the FF heuristic to avoid full grounding of the planning task. Our work builds on this connection between lifted planning and Datalog evaluation. While previous approaches use semi-naive evaluation to derive lifted planning algorithms, they evaluate rules sequentially. In contrast, we exploit the inherent parallelism of rule evaluation by developing a parallel semi-naive execution architecture for heuristic search with the FF heuristic.

\subsection{Hash-Distributed Search}

Parallel planning has been widely studied using distributed search algorithms. Hash-distributed A* and hash-distributed greedy best-first search \cite{kishimoto-et-al-icaps2009,kishimoto-fukunaga-botea-aaai2013} distribute states across workers according to a hash function, enabling parallel exploration of the search space while minimizing duplicate state generation. These approaches parallelize the search process itself. In contrast, our work focuses on parallelizing the lifted reasoning required for successor generation and heuristic computation. To our knowledge, hash-distributed search has not yet been explored in the context of lifted planning. Our approach is therefore complementary to distributed search techniques and can be combined with them to further increase parallel scalability.

\subsection{Parallel Datalog Evaluation}

Efficient evaluation of Datalog programs has been extensively studied in the database and knowledge representation communities. Semi-naive evaluation \cite{bancilhon-ramakrishnan-icmd-1986} is the standard technique for avoiding redundant rule instantiations during bottom-up evaluation. A number of systems have explored parallel and distributed Datalog execution, including both shared-memory and distributed systems \egcite{jordan-et-al-iccav2016,shkapsky-et-al-sigmod2016}. However, high-performance Datalog evaluation often relies on workload-specific optimizations and evaluation strategies. As a result, general-purpose Datalog engines are rarely directly applicable to classical planning workloads. Instead, planning systems typically implement specialized rule evaluation mechanisms tailored to the structure of planning domains. Our work follows this approach by developing a specialized semi-naive Datalog solver for lifted classical planning that exploits parallelism in both rule evaluation and grounding.





\section{Background}

\subsection{Classical Planning}

A \defined{classical planning task} is a tuple $\tup{\predicates,\actionschemas,\objects,\initial,\goal}$.
The \defined{planning domain} comprises $\tup{\predicates,\actionschemas}$, and the \defined{task information} comprises $\tup{\objects,\initial,\goal}$, namely:

\paragraph{Predicates and Literals.} $\predicates$ is a set of predicate symbols.
Each \defined{predicate symbol} $\predicate$ in $\predicates$ has an associated arity $\arity{\predicate}$ in $\naturals$.
An \defined{atom} over a predicate $\predicate$ of arity $\arity{\predicate} = k$ is an expression $\predicate(\variable_1,\ldots,\variable_k)$ where $\variable_i$ with $i=1,\ldots,k$ are variables or constants.
A \defined{literal} is either an atom $\atom$ or its negation $\neg \atom$.

\paragraph{Objects and Substitutions.} $\objects$ is a set of objects (constants).
A \emph{substitution function} $\substitution$ over some set of variables $\variables(\substitution)$ and objects $\objects$ is a function $\substitution : \variables(\substitution)\rightarrow\objects$ mapping each variable in $\variables(\substitution)$ to an object in $\objects$. We write $\substitutions(\variables,\objects)$ for the set of all such $\substitution$. We write $\variable_1/\object_1,\ldots,\variable_n/\object_n$ for a substitution function $\substitution$ over variables $\{\variable_1,\ldots,\variable_n\}$ and objects $\objects$ such that $\substitution(\variable_i) = \object_i$ for all $i=1,\ldots,n$. We write $\substitution\subseteq\substitution'$ iff $\variables(\substitution)\subseteq\variables(\substitution')$ and $\substitution(\variable) = \substitution'(\variable)$ for all $\variable\in\variables(\substitution)$. The application of $\substitution$ to a syntactic element $\element$ (e.g., an atom, literal, or action), denoted by $\element[\substitution]$, replaces every occurrence of each variable $\variable_i$ with the corresponding object $\object_i = \substitution(\variable_i)$. The set of free variables occurring in an element $\element$ is denoted $\variables(\element)$, and the arity $\arity{\element}$ is defined as $|\variables(\element)|$. We say that $\element$ is \defined{ground}, denoted by $\groundelement$, iff $\variables(\element)=\emptyset$.

\paragraph{States.} $\initial$ is the \defined{initial state}. A \defined{state} $\state$ consists of a set of ground atoms $\groundatoms(\state)$ that hold in $\state$. Ground atoms not in $\groundatoms(\state)$ are false. A positive ground literal $\groundatom$ \defined{holds} in $\state$ iff
$\groundatom\in\groundatoms(\state)$, and a negative ground literal $\neg\groundatom$ \defined{holds} iff
$\groundatom\notin\groundatoms(\state)$.

\paragraph{Goals.} $\goal$ is the goal, which is a set of ground literals.

\paragraph{Actions and Successors.} $\actionschemas$ is a set of action schemas. Each \defined{action schema} $\actionschema$ in $\actionschemas$ consists of two sets $\precondition{\actionschema}$ and $\effect{\actionschema}$, where $\precondition{\actionschema}$ is a set of precondition literals and $\effect{\actionschema}$ is a set of effect literals. We write $\pprecondition{\actionschema}$, $\mprecondition{\actionschema}$, $\peffect{\actionschema}$, and $\meffect{\actionschema}$ for the sets of atoms of positive and negative literals in $\precondition{\actionschema}$ and $\effect{\actionschema}$. A ground action $\groundaction$ is \defined{applicable} in a state $\state$ iff every ground literal $\groundliteral$ in $\precondition{\groundaction}$ holds in $\state$. Applying a ground action $\groundaction$ in a state $\state$ where it is applicable yields the \defined{successor state} $\state'$ with $\groundatoms(\state') = (\groundatoms(\state)\setminus\meffect{\groundaction})\cup\peffect{\groundaction}$.

The objective for a given classical planning task is to find a \defined{plan}, which is a sequence of ground actions that, when successively applied starting from the initial state, yields a state in which all ground literals mentioned in the goal hold.

\subsection{Substitution Consistency Graph}

\inlinecite{stahlberg-ecai2023} proposed clique enumeration as an alternative to database techniques \cite{correa-et-al-icaps2021} for computing applicable ground actions in lifted classical planning. Given a state, the approach computes a polynomial-size factored overapproximation of the variable substitutions satisfying the action preconditions, enumerates this overapproximation, and filters the resulting candidates by a final applicability check. Its running time is output-sensitive in the size of the explicitly represented overapproximation. In contrast, database-based techniques may suffer from the join-explosion problem, constructing large intermediate relations even when the final set of applicable ground actions is small \cite{ngo-et-al-jacm2018}. On the vast majority of classical planning benchmarks, clique enumeration is sound and complete, i.e., the enumerated candidates coincide with the set of applicable ground actions. Before stating the necessary conditions, we define the core data structure of this approach, the substitution consistency graph.

The \defined{substitution consistency graph} $\consistencygraph_{\actionschema,\state}$ for an action $\actionschema$ in a state $\state$ is a tuple $\tup{\vertices(\consistencygraph_{\actionschema,\state}), \edges(\consistencygraph_{\actionschema,\state})}$ where
$\vertices(\consistencygraph_{\actionschema,\state}) = \{\variable/\object \mid \variable\in\variables(\actionschema), \object\in\objects \}$ is the set of vertices, each representing a possible variable substitution, and
$\edges(\consistencygraph_{\actionschema,\state}) = \vertices(\consistencygraph_{\actionschema,\state})\times \vertices(\consistencygraph_{\actionschema,\state})\setminus (\removaledges^{\neq}\cup \removaledges^+\cup \removaledges^-)$ is the set of edges, each representing all possible pairs of variable substitutions \cite{stahlberg-ecai2023}.
An atom $\predicate(\variable_1,\ldots,\variable_n)$ \defined{matches} an atom $\predicate'(\variable_1',\ldots,\variable_n')$ iff $\predicate = \predicate'$
and for all $i=1,\ldots,n$, $\variable_i$ or $\variable_i'$ are variables, or else $\variable_i = \variable_i'$, i.e., both are objects.
The sets $\removaledges^{\neq}, \removaledges^+, \removaledges^-$ are:

\begin{itemize}

\item $\removaledges^{\neq} = \{\{\variable/\object_1,\variable/\object_2\} \mid \variable\in\variables(\actionschema), \object_1,\object_2\in\objects, \object_1 \neq \object_2\}$,
i.e., the set of edges whose corresponding substitution function would assign different objects to the same variable,

\item $\removaledges^+ = \{\{\vertex,\vertex'\}\mid \vertex,\vertex'\in \vertices(\consistencygraph_{\actionschema,\state}),\exists \atom\in\pprecondition{\actionschema},\forall \atom'\in\state~.~\atom[\vertex,\vertex'] \text{ does not match } \atom'\}$,
i.e., the set of edges for which some positive precondition atom, after applying the corresponding substitution, does not match any ground atom in the state.

\item $\removaledges^- = \{\{\vertex,\vertex'\}\mid \vertex,\vertex'\in \vertices(\consistencygraph_{\actionschema,\state}),\exists \atom\in\mprecondition{\actionschema}~.~\atom[\vertex,\vertex']\in\state\}$,
i.e., the set of edges for which some negative precondition atom, after applying the corresponding substitution, yields a ground atom that is contained in the state.

\end{itemize}

The following theorem, from \inlinecite{stahlberg-ecai2023}, establishes the conditions for precisely identifying all applicable ground actions via $k$-clique enumeration.

\begin{theorem}

Consider an action $\actionschema$ with arity $k\geq 2$ where each literal $\literal\in\precondition{\actionschema}$ has arity at most $2$, and a state $\state$. Then, there exists a $k$-clique $\clique = \{\vertex_1,\ldots,\vertex_k\}$ in $\consistencygraph_{\actionschema,\state}$ iff all ground literals $\groundliteral = \literal[\vertex_1,\ldots,\vertex_k]$ with $\literal\in\precondition{\actionschema}$ hold in $\state$.

\label{thm:classical:sound_and_complete}

\end{theorem}

Theorem~\ref{thm:classical:sound_and_complete} enables using $k$-clique algorithms to enumerate candidate applicable ground actions. If all literals in $\precondition{\actionschema}$ have arity at most $2$, every enumerated $k$-clique corresponds to a ground instance $\groundaction$ of $\actionschema$ that is applicable in $\state$. For actions containing literals of arity greater than $2$, we explicitly check whether those ground literals hold in $\state$. For actions of arity one, one can define a vertex removal criterion analogous to the edge removal criteria above. Moreover, actions of arity zero are already ground, and we can simply check their applicability, for instance using precondition match trees \cite{helmert-aij2009}.

\subsection{Datalog}

A \defined{Datalog program} (program) $\dprogram$ is a pair $\tup{\dfacts,\drules}$ where $\dfacts$ is a set of facts (ground atoms) and $\drules$ is a set of rules. Each rule $\drule$ in $\drules$ has the form $h\leftarrow b_1,\ldots,b_k$ where $k\geq 0$, $h$ is an atom (the ``head''), and $b_1,\ldots,b_k$ are literals (the ``body''). We write $\dhead(\drule)$ for the head and $\dbody(\drule)$ for the body of a rule $\drule$. Analogous to classical planning, a ground rule $\dgroundrule$ is \defined{applicable} in a fact set $\ddatabase$ if every literal $\groundliteral$ in $\dbody(\dgroundrule)$ holds in $\ddatabase$.

A program $\dprogram = \tup{\dfacts,\drules}$ with predicates $\predicates$ induces a \defined{predicate dependency graph} with a vertex for each predicate, and a directed edge $\tup{\predicate,\predicate'}$ if $\predicate$ occurs in the body and $\predicate'$ in the head of a rule $\drule \in \drules$, labeled $\top$ (resp.~$\bot$) if $\predicate$ occurs positively (resp.~negatively) in $\drule$.

A program is \defined{stratifiable} if its predicates can be partitioned into strata $\predicates_1,\ldots,\predicates_\ell$ such that for every edge $\tup{\predicate,\predicate'}$ in the dependency graph labeled $\top$ (resp.~$\bot$), we have $\predicate\in\predicates_i$ and $\predicate'\in\predicates_j$ with $i \le j$ (resp.~$i < j$). This stratification induces rule strata $\drules_1,\ldots,\drules_\ell$, where $\drule \in \drules_i$ if the predicate in $\dhead(\drule)$ belongs to $\predicates_i$.

The \defined{model} $\dmodel(\dprogram)$ of a stratifiable program $\dprogram$ is defined inductively over the rule strata $\drules_1,\ldots,\drules_\ell$. Let $\dfacts_0 \coloneqq \dfacts$. For each $i=1,\ldots,\ell$, we define the consequence operator
\[
\dconsequence_i(\ddatabase)
=
\dfacts_{i-1}
\;\cup\;
\{\, \dhead(\dgroundrule) \mid
\dgroundrule \in \ground(\drules_i, \ddatabase)
\,\},
\]
where $\ground(\drules_i, \ddatabase)$ is the set of all ground instances of rules in $\drules_i$ whose body is applicable in $\ddatabase$, i.e., $\ddatabase \vDash \dbody(\dgroundrule)$. Then $\dfacts_i$ is the least fixed point of $\dconsequence_i$, and $\dmodel(\dprogram) \coloneqq \dfacts_\ell$.

The least fixed point of $\dconsequence_i$ can be computed more efficiently using \defined{semi-naive bottom-up} evaluation \cite{bancilhon-ramakrishnan-icmd-1986}. For each stratum $i$, let $\ddatabase_i^0 \coloneqq \dfacts_{i-1}$ and $\ddelta_i^0 \coloneqq \dfacts_{i-1}$. For $\diter \ge 0$, define
\[
\ddelta_i^{\diter+1}
\coloneqq
\{\dhead(\dgroundrule) \mid
\dgroundrule \in \ground(\drules_i, \ddatabase_i^\diter),
\dbody(\dgroundrule) \cap \ddelta_i^\diter \neq \emptyset
\}
\setminus \ddatabase_i^\diter,
\]
and update $\ddatabase_i^{\diter+1} \coloneqq \ddatabase_i^\diter \cup \ddelta_i^{\diter+1}$.
Thus, in each iteration, only those rule instances are considered whose body contains at least one atom that was derived in the previous iteration. The evaluation terminates when $\ddelta_i^\diter = \emptyset$ for some $\diter$, and we set $\dfacts_i \coloneqq \ddatabase_i^\diter$.

\section{Datalog Compilations}

Many core components of lifted classical planning can be expressed as Datalog programs, including successor generation \cite{correa-et-al-icaps2020}, axiom evaluation, task grounding \cite{helmert-aij2009}, and delete-relaxed planning graph (RPG) heuristics \cite{correa-et-al-icaps2021,correa-et-al-aaai2022}. In this paper, we present Datalog compilations for applicable action generation and relaxed planning graph construction. Action generation underlies all state-space search methods, while RPG heuristics often provide strong search guidance. In the following, let $\tup{\predicates,\actionschemas,\objects,\initial,\goal}$ be a planning task and $\state$ a state.

\subsection{Applicable Action Program}

The \emph{applicable action program} is $\dprogram_\actionschemas = \tup{\dfacts,\drules_\actionschemas}$,
where $\dfacts$ is the state $\state$, and for each action schema $\actionschema \in \actionschemas$ with variables $\variables(\actionschema)$, positive precondition atoms $\pprecondition{\actionschema} = \{\atom_1,\ldots,\atom_m\}$, and negative precondition atoms $\mprecondition{\actionschema} = \{\atom_{m+1},\ldots,\atom_n\}$, there is a rule $\drule_\actionschema \in \drules_\actionschemas$ of the form
\[
\actionschema\text{-applicable}(\variables(\actionschema)) \leftarrow \atom_1,\ldots,\atom_m,\neg\atom_{m+1},\ldots,\neg\atom_n.
\]

Each ground instance of $\actionschema\text{-applicable}(\variables(\actionschema))$ in the model $\dmodel(\dprogram_\actionschemas)$
with substitution function $\substitution$ induces a ground action $\groundaction = \actionschema[\substitution]$ that is applicable in $\state$. Taking the union of these ground actions over all $\actionschema \in \actionschemas$ yields exactly the set of ground actions applicable in $\state$.


\subsection{Delete-Relaxed Planning Graph Program}

The \emph{RPG program} is $\dprogram_{\text{RPG}} = \tup{\dfacts,\drules_{\text{RPG}}}$,
where $\dfacts$ is the state $\state$, and for each action schema $\actionschema \in \actionschemas$ with variables $\variables(\actionschema)$, positive precondition atoms $\pprecondition{\actionschema} = \{\atom_1,\ldots,\atom_m\}$, and static negative precondition atoms $\msprecondition{\actionschema} = \{\atom_{m+1},\ldots,\atom_n\}$, $\drules_{\text{RPG}}$ contains, for each positive effect atom $\atom_e \in \peffect{\actionschema}$, the rule
\[
\atom_e \leftarrow \atom_1,\ldots,\atom_m,\neg\atom_{m+1},\ldots,\neg\atom_n.
\]

The restriction to static negative preconditions, i.e., atoms over predicates whose truth value never changes, is necessary to preserve monotonicity of the Datalog model. In contrast, this restriction is not needed for the applicable action program, since no rule head occurs in any rule body and semi-naive evaluation therefore terminates after one iteration. The cost of a ground atom in the head is defined by an aggregation function over the costs of the positive ground atoms in the body, e.g., summation for the additive heuristic or maximum for the max heuristic. For the FF heuristic, a relaxed plan is extracted by backward chaining from the goal atoms over supporting ground rule instances. The heuristic value is the number of distinct ground actions inducing these rule instances, and the preferred actions are those ground actions that are applicable in $\state$.

\section{Semi-naive Program Evaluation (\deltakpkc{})}

We now extend clique-based grounding from action schemas to semi-naive Datalog evaluation, where the main challenge is to enumerate exactly those rule instances whose bodies contain at least one newly derived fact.

The approach of \inlinecite{stahlberg-ecai2023} for generating ground applicable actions naturally extends to Datalog rules since the rule bodies are syntactically equivalent to action preconditions. However, semi-naive bottom-up evaluation requires enumerating each relevant $k$-clique exactly once within a monotonically growing fact set for each rule stratum. More precisely, for stratum $\drules_i$ and iteration $\diter$, we must identify all ground rules $\dgroundrule \in \ground(\drules_i,\ddatabase_i^\diter)$ with $\dbody(\dgroundrule)\cap\ddelta_i^\diter \neq \emptyset$, as required by the definition of $\ddelta_i^{\diter+1}$.

\begin{definition}[$\ddelta$-edges]
Consider a stratum $\drules_i$, two consecutive semi-naive iterations $\diter-1$ and $\diter$ with fact sets $\ddatabase_i^{\diter-1}$ and $\ddatabase_i^\diter$, and a rule $\drule\in\drules_i$. Let $\consistencygraph(\drule,\ddatabase_i^{\diter-1})$ and $\consistencygraph(\drule,\ddatabase_i^\diter)$ denote the substitution consistency graphs for rule $\drule$ under $\ddatabase_i^{\diter-1}$ and $\ddatabase_i^\diter$, respectively. The set of $\ddelta$-edges of $\drule$ in $\ddatabase_i^\diter$, is defined as
\[
\ddelta\edges_{i,\drule}^\diter
=
\edges(\consistencygraph(\drule,\ddatabase_i^\diter))
\setminus
\edges(\consistencygraph(\drule,\ddatabase_i^{\diter-1})).
\]
\end{definition}

$\ddelta$-edges arise from facts in $\ddelta_i^\diter$ that were added in the previous iteration. Hence, we can use $\ddelta$-edges to seed the $k$-clique enumeration. Since a clique may contain multiple $\ddelta$-edges, we must ensure that such duplicates are filtered out.

We adopt a branch-and-bound algorithm that combines an edge-oriented strategy \cite{wang-et-al-acmmd2024} with techniques for enumerating $k$-cliques in $k$-partite graphs \cite{phillips-et-al-algorithms2019}, building on the classical Bron-Kerbosch algorithm for clique enumeration \cite{bron-kerbosch-1973}. The main idea of edge-oriented strategies is to seed the clique-enumeration algorithm at an edge, thereby partially initializing the candidate clique with its incident vertices. A clique is reported only if its unique owner, defined as a canonical edge among the edges induced by the clique, coincides with the seeding edge \cite{wang-et-al-acmmd2024}. In our setting, we are interested in generating each relevant $k$-clique exactly once during semi-naive evaluation. Hence, we use the $\ddelta$-edges in $\ddelta\edges_{i,\drule}^\diter$ as seeds, since new edges may appear between two vertices that were already consistent. We define clique ownership for $\ddelta$-edges as follows.

\begin{definition}[Clique Owner]
Let $\rank$ be a fixed total order on edges. Consider a $k$-clique $\clique = \{\vertex_1,\ldots,\vertex_k\}$ with $k\geq 2$ that contains at least one edge from $\ddelta\edges_{i,\drule}^\diter$. The owner $\omega(\clique)$ is the edge $\tup{\vertex,\vertex'}\in \ddelta\edges_{i,\drule}^\diter$ with $\vertex,\vertex' \in \clique$ and minimum rank $\rank(\tup{\vertex,\vertex'})$.
\end{definition}

Alternative strategies avoid duplicate clique generation by imposing a canonical ordering on vertices \cite{wang-et-al-acmmd2024} or, in the $k$-partite case, on partitions. In contrast, our approach retains the flexibility to select pivot partitions dynamically, allowing the search space to be pruned during branch-and-bound by prioritizing partitions with the smallest candidate sets. The following lemma establishes the key property of $\ddelta$-edges, which guarantees that every clique corresponding to a ground rule with at least one newly supported body literal contains at least one $\ddelta$-edge.

\begin{lemma}[$\ddelta$-edge witness]
Let $\drule \in \drules_i$ and let $\bar r \in \ground(\{\drule\}, \ddatabase_i^\diter)$ be a ground rule such that $\dbody(\bar r)\cap \ddelta_i^\diter \neq \emptyset$.
Let $\clique(\bar r)$ be the corresponding $k$-clique in $\consistencygraph(\drule,\ddatabase_i^\diter)$. Then $\clique(\bar r)$ contains at least one edge from $\ddelta\edges_{i,\drule}^\diter$.
\label{lemma:delta-kpkc:edge-witness}
\end{lemma}

\begin{proof}
Since $\dbody(\bar r)\cap \ddelta_i^\diter \neq \emptyset$, some body literal of $\bar r$ corresponds to a fact in $\ddelta_i^\diter$, and is therefore present in $\ddatabase_i^\diter$ but not in $\ddatabase_i^{\diter-1}$. The clique $\clique(\bar r)$ encodes the ground substitution induced by the body of $\bar r$ in the substitution consistency graph. Hence, the consistency edge corresponding to that newly available body fact is present in $\consistencygraph(\drule,\ddatabase_i^\diter)$ but absent from $\consistencygraph(\drule,\ddatabase_i^{\diter-1})$. Consequently, $\clique(\bar r)$ contains an edge in
\[
\edges(\consistencygraph(\drule,\ddatabase_i^\diter))
\setminus
\edges(\consistencygraph(\drule,\ddatabase_i^{\diter-1}))
=
\ddelta\edges_{i,\drule}^\diter,
\]
i.e., at least one $\ddelta$-edge.
\end{proof}

\begin{theorem}[Correctness of $\ddelta$-edge anchored clique enumeration]
Let $\drules_i$ be a rule stratum and let $\diter \ge 0$. For each rule $\drule \in \drules_i$ with arity at least $2$, where each literal in $\dbody(\drule)$ has arity at most $2$, enumerate all $k$-cliques in $\consistencygraph(\drule,\ddatabase_i^\diter)$ that are seeded at some edge in $\ddelta\edges_{i,\drule}^\diter$, and report a clique only if its owner coincides with the seeding edge. Let $H_i^\diter$ be the set of heads of the corresponding ground rules. Then
\[
H_i^\diter
=
\{\dhead(\overline{r}) \mid
\overline{r} \in \ground(\drules_i, \ddatabase_i^\diter),
\dbody(\overline{r}) \cap \ddelta_i^\diter \neq \emptyset
\}.
\]
Consequently,
$\ddelta_i^{\diter+1}
=
H_i^\diter \setminus \ddatabase_i^\diter.$
\label{thm:delta-kpkc:correctness}
\end{theorem}

\begin{proof}
We show soundness, completeness, and uniqueness.

\textbf{Soundness.}
Consider a $k$-clique $\clique$ enumerated by the algorithm. By construction, $\clique$ is a clique in $\consistencygraph(\drule,\ddatabase_i^\diter)$, hence it corresponds to a consistent substitution satisfying all literals of $\dbody(\drule)$ in $\ddatabase_i^\diter$. Therefore, the corresponding ground rule $\bar r$ belongs to $\ground(\{\drule\},\ddatabase_i^\diter)$. Moreover, the clique enumeration is seeded at an edge in $\ddelta\edges_{i,\drule}^\diter$, so $\clique$ contains at least one $\ddelta$-edge. This implies that at least one body literal of $\bar r$ is supported by a fact in $\ddelta_i^\diter$, i.e.,
\[
\dbody(\bar r)\cap \ddelta_i^\diter \neq \emptyset.
\]
Thus $\dhead(\bar r)$ belongs to the set on the right-hand side of the equality.

\textbf{Completeness.}
Let $\bar r \in \ground(\{\drule\},\ddatabase_i^\diter)$ be a ground rule such that
\[
\dbody(\bar r)\cap \ddelta_i^\diter \neq \emptyset.
\]
By Lemma~\ref{lemma:delta-kpkc:edge-witness}, the corresponding clique $\clique(\bar r)$ contains at least one edge in $\ddelta\edges_{i,\drule}^\diter$. Therefore, the algorithm will consider $\clique(\bar r)$ when seeding the enumeration from that edge. Consequently, the clique corresponding to $\bar r$ will be generated.

\textbf{Uniqueness.}
A clique may contain multiple $\ddelta$-edges and could therefore be discovered multiple times. However, a clique is reported only if its owner edge coincides with the seeding edge. Since $\rank$ defines a total order on edges, every clique has a unique owner edge. Hence, each clique is reported exactly once.

Combining soundness, completeness, and uniqueness, $H_i^\diter$ equals
\[
\{\dhead(\overline{r}) \mid
\overline{r} \in \ground(\drules_i, \ddatabase_i^\diter),
\dbody(\overline{r}) \cap \ddelta_i^\diter \neq \emptyset
\}.
\]
By the definition of semi-naive evaluation, $\ddelta_i^{\diter+1}$ is obtained by removing already known facts, yielding
\[
\ddelta_i^{\diter+1} = H_i^\diter \setminus \ddatabase_i^\diter.\qedhere
\]
\end{proof}

For rules containing body literals of arity greater than $2$, these literals must be verified explicitly for each candidate ground instance. If a candidate is inapplicable due to such a higher-arity literal, it must be queued and rechecked in every subsequent iteration, since $\ddelta$-edge seeding provides no guarantee of regenerating it.

\Omit{
\subsection{Time and Space Efficient Edge Updates}

In this section, we present structural exploitation on the consistency graph, allowing us to speed up its construction for a given state,
while also compressing its memory representation. The variable dependency graph of an action schema $\actionschema$ is a graph $\vdg_{\actionschema} = (\vertices(\vdg_{\actionschema}), \edges(\vdg_{\actionschema}))$ consisting of
\begin{itemize}
    \item $\vertices(\vdg_{\actionschema}) = \{\variable_i\mid 1\leq i\leq\arity{\actionschema}\}$
    \item $\edges(\vdg_{\actionschema}) = \{\{\variable_i,\variable_j\}\mid \variable_i,\variable_j\in \vertices(\vdg_{\actionschema})\land \exists p\in\precondition{\actionschema}\text{ that mentions } \variable_i \text{ and } \variable_j \}$
\end{itemize}

\begin{lemma}
    Consider $\variable$ and $\variable'$ in $\vertices(\vdg_{\actionschema})$ with $\{\variable,\variable'\}\notin \edges(\vdg_{\actionschema})$.
    If $\variable/\object$ and $\variable'/\object'$ in $\vertices(\consistencygraph_{\actionschema,\state})$ (vertex-consistent),
    then it follows that $\{\variable/\object,\variable'/\object'\}\in \edges(\consistencygraph_{\actionschema,\state})$ (edge-consistent).
    \label{lemma:cc_edge_check_simplification}
\end{lemma}

\begin{proof}
    Assume $\variable/\object,\variable'/\object'\in \vertices(\consistencygraph_{\actionschema,\state})$ (vertex-consistent),
    and there is no edge between $\variable$ and $\variable'$ in $\vdg_\actionschema$.
    We must show that $\{\variable/\object,\variable'/\object'\}$ does not make it into any edge exclusion criteria
    in the definition of the substitution consistency graph:
    \begin{itemize}
        \item $\{\variable/\object,\variable'/\object'\}\notin\mathcal{I}^{\neq}$
        because $\variable\neq \variable'$,
        \item $\{\variable/\object,\variable'/\object'\}\notin\mathcal{I}^+$
        because there does not exist a $\atom\in\pprecondition{\actionschema}$
        such that $\forall \atom'\in\state$, $\atom[\variable/\object,\variable'/\object']$ does not match $\atom'$, and
        \item $\{\variable/\object,\variable'/\object'\}\notin\mathcal{I}^-$
        because there does not exist a $\atom\in\mprecondition{\actionschema}$
        such that $\groundatom[\variable/\object,\variable'/\object']\in\state$.
    \end{itemize}
    Consequently it follows that $\{\variable/\object,\variable'/\object'\}\in\edges(\consistencygraph_{\actionschema,\state})$ (edge-consistent).
\end{proof}

Lemma~\ref{lemma:cc_edge_check_simplification} allows us to check edge consistency purely based on vertex consistency,
which is much faster since the number of vertices between two partitions is $O(|\objects|^2)$ while the number of vertices in a partition is $O(|\objects|)$.
The following Corollary~\ref{lemma:cc_edge_check_redundance} reveals that we can derive adjacency structure from previous nodes $\variable/\object\in\vertices(\consistencygraph_{\actionschema,\state})$.

\begin{corollary}
    Consider $\variable$ and $\variable'$ in $\vertices(\vdg_{\actionschema})$ with $\{\variable,\variable'\}\notin\edges(\vdg_\actionschema)$,
    and two vertices $\variable/\object$ and $\variable/\object'$ in $\vertices(\consistencygraph_{\actionschema,\state})$ (vertex-consistent).
    For all $\variable'/\object''$ in $\vertices(\consistencygraph_{\actionschema,\state})$ (vertex-consistent) it holds that
    \[
    \{\variable/\object,\variable'/\object''\}\in \edges(\consistencygraph_{\actionschema,\state})
    \iff
    \{\variable/\object',\variable'/\object''\}\in \edges(\consistencygraph_{\actionschema,\state}).
    \]
    \label{lemma:cc_edge_check_redundance}
\end{corollary}

\begin{proof}
    The claim follows immediately from Lemma~\ref{lemma:cc_edge_check_simplification}.
\end{proof}

\subsection{Exploiting Database Techniques}

The edge update in the previous section can be mapped into relational database operations involving at most two variables at a time.
In contrast to general database techniques, this strictly avoids intermediate blowups in the interpretation of relations,
at the expense of encoding only literals of arity at most two exactly, while approximating higher arity literals.

We can compute the consistent edges via database projection to binary relation, followed by combinations of intersection, union, and difference on binary relations.
Consider a set of relations $\mathcal{R}$ each $R$ in $\mathcal{R}$ has an associated arity $\arity{R}$, and a universe $\objects$.
The interpretation of a relation $R$ in $\mathcal{R}$ is $R^I\subseteq \objects^{\arity{R}}$.

\begin{definition}[Projection]
Let $R \in \mathcal{R}$ be a relation with arity $\arity{R}$ and interpretation
$R^I \subseteq \objects^{\arity{R}}$.
Let $I = (i_1,\dots,i_k)$ be a tuple of distinct argument positions with
$i_\ell \in \{1,\dots,\arity{R}\}$ for all $\ell \in \{1,\dots,k\}$.
The projection of $R^I$ onto the positions $I$, denoted $R[I]$, is defined as
\[
R[I]
=
\{
(\object_{i_1},\dots,\object_{i_k})
\mid
(\object_1,\dots,\object_{\arity{R}}) \in R^I
\}.
\]

\end{definition}

\begin{definition}[Intersection]
Let $R,S \in \mathcal{R}$ with interpretations
$R^I, S^I \subseteq \objects^k$ for some $k\in\mathbb{N}$.
The intersection of $R^I$ and $S^I$ is
\[
R^I \cap S^I
=
\{ \mathbf{\object} \in \objects^k \mid \mathbf{\object} \in R^I \land \mathbf{\object} \in S^I \}.
\]
\end{definition}

\begin{definition}[Union]
Let $R,S \in \mathcal{R}$ with interpretations
$R^I, S^I \subseteq \objects^k$ for some $k\in\mathbb{N}$.
The union of $R^I$ and $S^I$ is
\[
R^I \cup S^I
=
\{ \mathbf{\object} \in \objects^k \mid \mathbf{\object} \in R^I \lor \mathbf{\object} \in S^I \}.
\]
\end{definition}

\begin{definition}[Difference]
Let $R,S \in \mathcal{R}$ with interpretations
$R^I, S^I \subseteq \objects^k$ for some $k\in\mathbb{N}$.
The difference of $R^I$ and $S^I$ is
\[
R^I \setminus S^I
=
\{ \mathbf{\object} \in \objects^k \mid \mathbf{\object} \in R^I \land \mathbf{\object} \notin S^I \}.
\]
\end{definition}

\begin{definition}[Edge Consistency Relation]
Let $\actionschema$ be an action schema with variable dependency graph $\vdg_{\actionschema}$,
and let $\mathcal{R}^+$ and $\mathcal{R}^-$ be the sets of relations appearing in
$\pprecondition{\actionschema}$ and $\mprecondition{\actionschema}$, respectively.
For variables $\variable,\variable' \in \vertices(\vdg_{\actionschema})$, define
$\mathcal{R}_{\variable,\variable'}^{\circ} =\{ R \in \mathcal{R}^{\circ} \mid R \text{ mentions } \variable \text{ and } \variable' \}$ for $\circ\in\{+,-\}$.
The edge consistency relation between $\variable$ and $\variable'$ is
\[
\mathit{ER}(\variable,\variable')
=
\left(
\bigcap_{R \in \mathcal{R}_{\variable,\variable'}^+} R[\variable,\variable']
\right)
\setminus
\left(
\bigcup_{R \in \mathcal{R}_{\variable,\variable'}^-} R[\variable,\variable']
\right),
\]
where the intersection over an empty set equals $\objects\times\objects$,
and the union over an empty set equals $\emptyset$.
\end{definition}

\begin{theorem}
    Consider $\variable/\object$ and $\variable'/\object'$ in $\vertices(\consistencygraph_{\actionschema,\state})$.
    Then, $\{\variable/\object,\variable'/\object'\}\in\edges(\consistencygraph_{\actionschema,\state})$
    iff $(\object,\object')\in\mathit{ER}(\variable,\variable')$.
\end{theorem}

\begin{proof}
    ($\Rightarrow$)
    Assume $\{\variable/\object,\variable'/\object'\}\in\edges(\consistencygraph_{\actionschema,\state})$.
    By definition of the consistency graph, none of the exclusion criteria
    $\mathcal I^{\neq}$, $\mathcal I^+$, or $\mathcal I^-$ applies.

    Thus, for every positive precondition atom mentioning both $\variable$ and $\variable'$,
    there exists a matching ground atom in $\state$, implying $(\object,\object')\in R[\variable,\variable']$
    for the corresponding relation $R$.
    Similarly, no negative precondition excludes $(\object,\object')$.
    Therefore, $(\object,\object')$ belongs to every projected relation, and hence
    $(\object,\object')\in\mathit{ER}(\variable,\variable')$.

    ($\Leftarrow$)
    Assume $(\object,\object')\in\mathit{ER}(\variable,\variable')$.
    Then $(\object,\object')$ satisfies all projected relations corresponding to precondition
    atoms mentioning both $\variable$ and $\variable'$.
    Since $\variable/\object$ and $\variable'/\object'$ are vertex-consistent, none of the exclusion criteria
    $\mathcal I^{\neq}$, $\mathcal I^+$, or $\mathcal I^-$ applies.
    Thus $\{\variable/\object,\variable'/\object'\}\in\edges(\consistencygraph_{\actionschema,\state})$.
\end{proof}
}

\section{Parallelized Architecture}

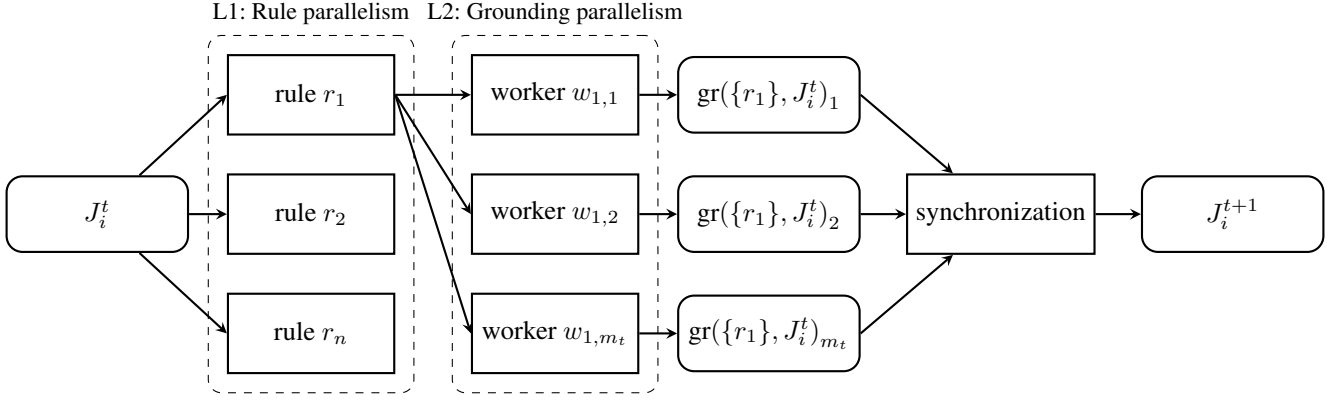
\begin{figure*}
\centering
\begin{tikzpicture}[
    >=stealth,
    node distance=0.5cm and 0.5cm,
    data/.style={
        draw,
        rounded corners=6pt,
        minimum width=2.4cm,
        minimum height=1cm,
        align=center,
        thick
    },
    operation/.style={
        draw,
        rectangle,
        minimum width=2.2cm,
        minimum height=1.05cm,
        align=center,
        thick
    },
    parallel/.style={
        draw,
        dashed,
        inner sep=7pt,
        rounded corners=5pt
    },
    arrow/.style={->, thick},
    label/.style={font=\small, inner sep=2pt}
]

\node[data] (Fi) {$\ddatabase_i^t$};

\node[operation, right=of Fi] (r2) {rule $\drule_2$};
\node[operation, above=of r2] (r1) {rule $\drule_1$};
\node[operation, below=of r2] (rk) {rule $\drule_n$};

\node[parallel, fit=(r1)(rk)] (rulepar) {};
\node[label, above=2pt of rulepar.north] {L1: Rule parallelism};

\draw[arrow] (Fi) -- (r1.west);
\draw[arrow] (Fi) -- (r2.west);
\draw[arrow] (Fi) -- (rk.west);

\node[operation, right=1cm of r1] (w1) {worker $\dworker_{1,1}$};
\node[operation, below=of w1] (w2) {worker $\dworker_{1,2}$};
\node[operation, below=of w2] (wk) {worker $\dworker_{1,{m_t}}$};

\node[parallel, fit=(w1)(wk)] (workerpar) {};
\node[label, above=2pt of workerpar.north] {L2: Grounding parallelism};

\draw[arrow] (r1.east) -- (w1.west);
\draw[arrow] (r1.east) -- (w2.west);
\draw[arrow] (r1.east) -- (wk.west);

\node[data, right=of w1] (g1) {${\ground(\{\drule_1\}, \ddatabase_i^t)}_1$};
\node[data, right=of w2] (g2) {${\ground(\{\drule_1\}, \ddatabase_i^t)}_2$};
\node[data, right=of wk] (gk) {${\ground(\{\drule_1\}, \ddatabase_i^t)}_{m_t}$};

\draw[arrow] (w1) -- (g1);
\draw[arrow] (w2) -- (g2);
\draw[arrow] (wk) -- (gk);

\node[operation, right=0.6cm of g2] (merge) {synchronization};

\node[data, right=0.6cm of merge] (Fnext) {$\ddatabase_i^{t+1}$};

\draw[arrow] (g1.east) -- (merge);
\draw[arrow] (g2.east) -- (merge);
\draw[arrow] (gk.east) -- (merge);

\draw[arrow] (merge) -- (Fnext);


\end{tikzpicture}
\caption{Synchronous execution architecture with two-level parallelization. Level~L1 denotes rule-level parallelism, where rules are evaluated concurrently,
and Level~L2 denotes grounding parallelism, where workers ground instances of a rule in parallel. Let $\dprogram = \tup{\dfacts,\drules}$ be a Datalog program with rule strata $\drules_1,\ldots,\drules_\ell$. During semi-naive evaluation, rule stratum $\drules_i = \{\drule_1,\ldots,\drule_n\}$ is processed with current fact set $\ddatabase_i^t$. For each rule $\drule \in \drules_i$, a rule task grounds all newly applicable instances $\ground(\{\drule\}, \ddatabase_i^t)$. Each rule task $r$ creates $m_t \ge 1$ worker tasks that ground disjoint subsets, i.e., $\ground(\{\drule\}, \ddatabase_i^t) = {\ground(\{\drule\}, \ddatabase_i^t)}_1 \uplus \cdots \uplus {\ground(\{\drule\}, \ddatabase_i^t)_{m_t}}$. Finally, a synchronization phase merges the heads of all derived ground rules into the fact set $\ddatabase_i^t$ to obtain $\ddatabase_i^{t+1}$.
}
\label{fig:architecture}
\end{figure*}

The $\ddelta$-edge anchored enumeration procedure from the previous section provides the rule-level work units needed for semi-naive evaluation; we now show how these work units are scheduled in a synchronous two-level parallel architecture. Figure~\ref{fig:architecture} shows the synchronous execution architecture with two-level parallelization where the coordination of facts happens after each iteration. Level~L1 is the primary driver of parallelism, evaluating rules concurrently. Depending on the structure of the input program, the per-rule execution time may vary greatly, or the number of rules may be too small, limiting parallel scaling. Therefore, level~L2 parallelizes the grounding phase of each rule by distributing $\ddelta$-edges among worker tasks, where each worker emits only those cliques that it owns according to the seeding edge. Choosing the number of workers $m_\diter$ is the responsibility of the load balancer, which aims to minimize the sequential execution time of iteration $\diter$. To reduce contention in memory allocation, all data structures used during search are stored in geometrically growing flat byte arrays that serialize dynamically sized objects.

\section{Experiments}

We implemented our planner, Tyr, in C++ with Python bindings. It uses the Loki PDDL parser and normalizer \cite{drexler-zenodo2026}, which implements the PDDL normalization described in Section~4 of \inlinecite{helmert-aij2009}. We will make Tyr available online. We use $N \in \{1,2,4,8\}$ rule-level workers, denote the corresponding configuration by Tyr-$N$, and use $m_\diter = 1$ grounding worker by default. For $N = 8$, we also consider the configuration Tyr-8-2, which uses $m_\diter = 2$ grounding workers whenever more than 1024 $\ddelta$-edges are generated, distributing them equally in round-robin fashion. More sophisticated load-balancing techniques are beyond the scope of this paper. As a lifted baseline, we consider the Powerlifted (PL) planner \cite{correa-et-al-aaai2022}, and as a ground baseline, we consider the Fast Downward (FD) planner \cite{helmert-jair2006}. All planners use greedy best-first search with the $h_{\text{FF}}$ heuristic and alternate between preferred and non-preferred queues \cite{richter-helmert-icaps2009}. We increase the preferred-queue weight by 1000 whenever the current best heuristic estimate improves, and decrease it by 1 after each node expansion. We evaluate all planners on the Autoscale Agile (AS) and Hard-To-Ground (HTG) benchmark sets, using a time limit of 10\,minutes and a memory limit of 16\,GiB.

\subsection{Speedup Analysis}

Parallelizing Datalog evaluation can improve planner performance only if two conditions are met: the sequential implementation must incur little avoidable overhead, and a sufficiently large fraction of the remaining runtime must lie in parallelizable phases. We therefore first analyze Tyr's runtime structure before comparing its end-to-end performance against the baselines. This analysis isolates the components affected by rule-level and grounding-level parallelism and derives empirical speedup limits from the observed time distribution.

We focus on the HTG benchmark set, whose tasks are challenging for ground planners and exhibit substantial Datalog execution times, making them the primary target for parallelization. In particular, we address the following central questions:

\begin{itemize}
    \item[Q1] How much of the total time is effectively spent in the Datalog execution of our planner?
    \item[Q2] Is our single-core implementation sufficiently efficient to motivate its parallelization?
    \item[Q3] What are the limits of the speedup given our empirical data?
    \item[Q4] What are the current weaknesses of our approach and how can we address them?
\end{itemize}

\subsubsection{Datalog Fraction.} The Datalog fraction, denoted by $\delta_D$, is defined as $T_D / T_\text{tot}$, where $T_D$ is the total wall-clock time spent in Datalog execution and $T_\text{tot}$ is the total wall-clock time of the planner execution.

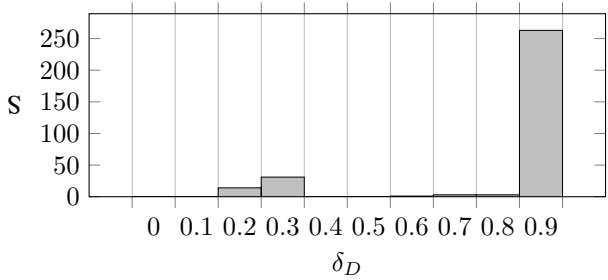
\begin{figure}[ht]
\begin{tikzpicture}
  \begin{axis}[
    ybar interval,
    xlabel={$\delta_D$},
    ylabel={S},
    ylabel style={rotate=90},
    width=\linewidth,
    height=4cm,
    ymin=0,
    ytick distance=50,
  ]
    \addplot[
      hist={bins=10, data min=0, data max=1}, thick,
      fill=black!25,
      draw=black,
      line width=0.3pt,
    ] table[y index=0] {
      data
    0.9702148281263726
    0.969510605207308
    0.976685511842544
    0.973049543499727
    0.9688439405745926
    0.9731651084464837
    0.9787595063575907
    0.9764269483756355
    0.9805635413509196
    0.9765972638304672
    0.9806492867035554
    0.9757474149191698
    0.9816132639715063
    0.9780440885851809
    0.9787355790600616
    0.982222741143879
    0.9771140270000517
    0.980696740457291
    0.9795953595088953
    0.9538118408496011
    0.9587238144334729
    0.9556496469407021
    0.9581825759423034
    0.9621183243018867
    0.9629011334283155
    0.962650756374254
    0.9690793557129092
    0.9712964600918607
    0.9705640753235114
    0.23275240388035465
    0.22858156309647373
    0.2519704375157242
    0.24002853846595038
    0.22246608230086487
    0.2615457758945935
    0.22759887691367958
    0.21552028713562837
    0.2688502469213443
    0.6820489877055063
    0.25648056883040793
    0.23044190767937056
    0.2184988693459092
    0.2947865123833793
    0.26132617974139627
    0.300989667696882
    0.32289646506284125
    0.35208721067606447
    0.34740597144386764
    0.33970711636903744
    0.3304797940461565
    0.34310374259151294
    0.3474139013690852
    0.3539936518208284
    0.33972687741855995
    0.3423287141102971
    0.3544772768006553
    0.36181774334529254
    0.3314287732747404
    0.3453255880244078
    0.36980045109190224
    0.33932961062497
    0.32419940551276916
    0.3282095615389478
    0.3080898407017723
    0.32434285462633883
    0.3227750004950169
    0.3523489458241468
    0.3787984247213634
    0.3524702985009616
    0.3630175305186342
    0.36915962756264964
    0.3915227236362605
    0.3748352379280699
    0.3620625041083106
    0.3680376135645712
    0.93177154952582
    0.9295790165743033
    0.9428455949713694
    0.9308994855915982
    0.9386361936962715
    0.9285704601359577
    0.9102063326690709
    0.9336583643043089
    0.9232738144744921
    0.9344117277855769
    0.9474971450034293
    0.9424350822087643
    0.9536623155794092
    0.9322117945145371
    0.9394523714372482
    0.9517748254355941
    0.9539921910875157
    0.9295958433949996
    0.9228591954852082
    0.930455904293547
    0.9444216463496675
    0.9254449711655323
    0.9510380677345153
    0.9231788121172525
    0.9264266880248195
    0.9479435400329392
    0.9409235258374938
    0.9531425805752706
    0.90908982370281
    0.9428496818501138
    0.9321875220874414
    0.9353947640082539
    0.9258799695917903
    0.9388109593466211
    0.9298081150652203
    0.9606440497065345
    0.9040745616480246
    0.9562163082514705
    0.9740414370195548
    0.981089015683296
    0.9867719265267398
    0.9691932369421308
    0.974485987466541
    0.9833186832527475
    0.9686431076877354
    0.987971425193179
    0.9818347031554875
    0.9835517274268286
    0.9798397027223715
    0.9861304451753815
    0.9836444786868036
    0.9821407719224898
    0.9848504726332872
    0.9853677151739353
    0.9817947637607201
    0.9872034920114137
    0.9847653244876114
    0.9768664164736084
    0.9770215470896103
    0.987551348082728
    0.9862544267996237
    0.9854766943786759
    0.9876714368763126
    0.9859610469204738
    0.9873361250447439
    0.9881712099343973
    0.9873844475348413
    0.9780125162844163
    0.9917901398771967
    0.9907941256289352
    0.9392158713032501
    0.9804312928359861
    0.9760643798497077
    0.9687154076757324
    0.9757029982003389
    0.9823202859179767
    0.9781663604534554
    0.9718076351936048
    0.9797994742109999
    0.9854800105600475
    0.9818571605398129
    0.9855000790574903
    0.9916476278519757
    0.9755462193721459
    0.9805664574918582
    0.9851363398212085
    0.980813348286979
    0.9679101223852473
    0.9684532945929252
    0.9841585326081789
    0.9686919934220858
    0.9821188345003481
    0.9838254930528334
    0.963744947961114
    0.9806706732126617
    0.9663222011140361
    0.9882728021298258
    0.9887502671110595
    0.9889823543602101
    0.984312335299733
    0.9779705849610695
    0.9827159497793957
    0.9718818588875289
    0.9839263728783167
    0.9857449865811148
    0.9895644128426792
    0.9825470495894628
    0.9815397707767956
    0.9767838983026421
    0.9687588494348902
    0.9826951835839843
    0.9827543244020359
    0.9839539294185414
    0.979647082722975
    0.9878843313747244
    0.9760883945657669
    0.9731830510904453
    0.9729876065372605
    0.9843511131536801
    0.9833003808280946
    0.9854542107443746
    0.9935481899165529
    0.9780697525531462
    0.9750083844208945
    0.975597533908778
    0.9746448208236699
    0.9781050569946439
    0.9866165190339979
    0.9737885299587155
    0.9583551625136527
    0.9474727280908816
    0.9854248817955981
    0.9874703069784653
    0.9858394093363069
    0.9843279497468673
    0.9818235394492499
    0.9904551745905538
    0.9847597703467734
    0.9555141485459935
    0.9527706259468497
    0.9385148387799581
    0.9500838793570782
    0.9405475051837893
    0.9517614215903472
    0.9351749181411658
    0.9522575808478813
    0.9553567608745221
    0.9241005445388263
    0.9466898403506182
    0.9594884258217969
    0.9828256586122949
    0.9034100093052296
    0.9684435796212525
    0.9615887842385195
    0.9992741294132411
    0.95090602086508
    0.9954982944060983
    0.9969499894768362
    0.9964743488561632
    0.9967222571161314
    0.998707925557997
    0.9968794108287353
    0.996701171756793
    0.9976532459243314
    0.9977864062102315
    0.9966484053814907
    0.9995005270897335
    0.9990129761416965
    0.9967823447468879
    0.9979521231544216
    0.9985127234295493
    0.9996228434906946
    0.9990157000640215
    0.905146109568233
    0.9705368928857685
    0.9651254619227222
    0.969880045906768
    0.8982159853359585
    0.9569082123257429
    0.9691971112346784
    0.8988437203978183
    0.9099351151271471
    0.973224196821004
    0.9807119697528467
    0.9246218102340088
    0.9106363806504207
    0.9473246308518353
    0.8048840047319289
    0.9464537810487382
    0.9502790066665637
    0.9663966613503263
    0.7632853815782267
    0.9264146818695406
    0.9621523170454097
    0.9722774056828738
    0.9303547909832601
    0.9472125515990962
    0.9677575868654603
    0.76639358252888
    0.9244836170903057
    0.9453237107037511
    0.9756183188229972
    0.7865672355975036
    0.9609783139292807
    0.9755403019924573
    0.9781014104815814
    0.991506720690877
    0.9250936371672585
    0.9984251385122912
    0.998493497904662
    0.9989263110437852
    0.9989088879422937
    0.9967628443466188
    0.9968795194860051
    0.9985275812578129
    0.9984799276172416
    0.9988657737595883
    0.9929186158815814
    0.9961342702310693
    0.9979251707339862
    0.9983853749800302
    0.9832732485505518
    0.9876172521157819
    0.990476140359711
    0.991317827263359
    0.9976724107262825
    0.9987323026280517
    0.994785231778367
    0.9968180810153745
    0.9927066316195375
    0.9967356755957008
    0.9847239173441177
    0.9851460164389017
    0.9855844123502954
    0.9973938832784387
    0.9972866777293496
    0.9974995348897475
    0.9974986732799653
    0.9962716370786961
    0.9974981408173274
    };
  \end{axis}
\end{tikzpicture}
  \caption{The number of solved tasks (S) grouped by their total wallclock time spent in the Datalog execution ($\delta_D$).}
  \label{fig:histogram}
\end{figure}

Figure~\ref{fig:histogram} shows the distribution of solved tasks grouped by $\delta_D$ for tasks whose total wall-clock time is at least 6 seconds (1\% of the time limit). Among the 315 such tasks, 45 have $\delta_D \in [0, 0.5)$ and 270 have $\delta_D \in [0.5, 1]$. Roughly 250 tasks fall into the interval $\delta_D \in [0.9, 1]$. A large Datalog fraction suggests substantial potential for parallelization. For instance, if $\delta_D = 0.9$, Amdahl's law bounds the maximum achievable speedup by $1 / (1 - 0.9) = 10$. Even with perfect parallelization of the Datalog component, the achievable speedup is therefore limited by the remaining sequential fraction. Before analyzing these speedup limits in more detail, we first evaluate the efficiency of the single-core implementation to justify parallelization.

\subsection{Search Time Per Expanded Node}

\input{scatterplot-search-time-ms-per-expanded.tex}

To assess whether our single-core implementation is sufficiently efficient, Figure~\ref{fig:efficiency} compares the search time per expanded node of PL and Tyr-1. For tasks with $\delta_D\in[0.5,1]$, the picture is largely balanced: neither configuration dominates, and performance differences are task-dependent. This indicates that the single-core implementation is sufficiently efficient to motivate parallelization.

\subsection{Empirical Speedup Limits}

Given empirical data from the single-core implementation and measurements separating strictly sequential and parallelizable execution phases, we now provide average empirical speedup limits. Table~\ref{tab:distribution} shows the fraction of total wall-clock time spent in the Action ($\delta_\text{Action}$) and FF ($\delta_\text{FF}$) Datalog evaluation phases. Across all tasks ($\delta_D\in[0,1]$), the Action component accounts for only 4.6\% of the total wall-clock time, while FF accounts for 83.1\%. For $\delta_D\in[0.5,1]$, where parallelization is feasible, FF reaches 97.6\%. Moreover, the parallelizable fraction within the FF phase, i.e., the share spent in rule evaluation as opposed to the sequential merge phase ($\delta_\text{FF}^\text{inter}$), is 94.7\%. Consequently, the overall parallel fraction for rule parallelization is $0.976 \times 0.947 \approx 0.924$, since only Datalog execution is parallelized. According to Amdahl's law, the maximum speedup is therefore bounded by $1/(1-0.924)\approx 13$. The intra-phase parallel fraction ($\delta_\text{FF}^\text{intra}$) of 77.2\% further bounds intra-rule speedups through load balancing by a factor of $1/(1-0.772)\approx 4.39$, which would be required to attain average empirical speedup limits of $13$.

\setlength{\tabcolsep}{3pt}

\begin{table}[ht]
  \centering
  \begin{tabular}{l r @{\hspace{0.5cm}} rrr rrr}
    \toprule
    & & \multicolumn{3}{c}{Action} 
      & \multicolumn{3}{c}{FF} \\
      \cmidrule(l){3-5}
      \cmidrule(l){6-8}
    $\delta_D$ 
    & N 
    & $\delta_\text{phase}^\text{seq}$ 
    & $\delta_\text{phase}^\text{inter}$ 
    & $\delta_\text{phase}^\text{intra}$  
    & $\delta_\text{phase}^\text{seq}$ 
    & $\delta_\text{phase}^\text{inter}$ 
    & $\delta_\text{phase}^\text{intra}$ \\        
    \midrule
    $[0,0.5)$ & 45  & 0.264 & 0.543 & 0.999 & 0.068 & 0.999 & 0.999   \\
    $[0.5,1]$ & 270 & 0.004 & 0.779 & 0.407 & 0.976 & 0.947 & 0.772   \\
    $[0,1]$   & 315 & 0.046 & 0.559 & 0.944 & 0.831 & 0.949 & 0.775   \\
    \bottomrule
  \end{tabular}
  \caption{Fraction of total wall-clock time spent in each execution phase, defined as $\delta_{\text{phase}} = \frac{\sum_i T_{\text{phase}, i}}{\sum_i T_{\text{total}, i}}$, where the sums range over all tasks in the respective $\delta_D$ interval.}
  \label{tab:distribution}
\end{table}

\subsection{Observed Empirical Speedups}

\begin{figure}[ht]
\begin{tikzpicture}
\begin{axis}[
    height=2.3in,
    legend cell align=left, 
    title={Rule Skew vs.~Relative Speedup (1 vs 8 cores)}, 
    title style={
        xshift=1.5cm
    },
    width=2.3in,  
    xlabel={Rule Skew}, 
    xmin=1,
    xmax=1000, 
    xtickten={0,1,2,3},
    xmode=log, 
    ylabel={Relative Speedup},
    ymin=0,
    ymax=8, 
    legend pos=south east,
    legend style={
        at={(1.05,1)},
        anchor=north west,
        font=\tiny,
        align=left,
    },
    label style={font=\scriptsize},
    tick label style={font=\scriptsize},
]

\addplot[only marks,color=butter3,mark=+] coordinates {
(73.15, 1.115060213405559) (27.68, 0.8862195053049898) (58.04, 1.2684104812722854) (45.42, 1.0813899271714824) (30.91, 1.0865501615829263) (43.14, 0.7508361353034063) (45.35, 1.0839701004812146) (55.72, 0.991966990531172) (39.71, 1.0429979030335927) (41.78, 0.9921543259769369) (38.59, 0.9911486895560986) (50.35, 0.9965212534309263) (45.81, 1.258713187417213) (54.64, 1.0536080724768824) (49.09, 0.8698255305702771) (45.84, 1.0054980160020883) (36.33, 1.088567872361323) (51.42, 1.0577827780055449) (36.6, 0.9166435204315532) (58.9, 1.072130699365932) (59.95, 1.1113713585499703) (47.84, 1.306513145504328) (64.98, 1.3195975998759288) (119.65, 0.7016366017273277) (102.47, 0.7993055648995242) (21.46, 1.037593154801589) (39.98, 0.9704732077717463) (47.07, 1.0089959953330938) (48.34, 0.9169630509596572) 
};
\addlegendentry{blocksworld-large}

\addplot[only marks,color=orange3,mark=+] coordinates {
(2.4, 2.0363251749101856) (2.63, 2.124796575578972) (2.31, 5.588853273368229) (1.84, 1.0569960700186307) (1.75, 1.3346695004629625) (2.89, 0.9513699050202781) (3.14, 1.805331370050506) (2.34, 0.7979906461348365) (2.93, 0.892275088119487) (1.65, 1.5178695742619208) (3.9, 2.4807430314533727) (2.08, 1.0979858644918283) (3.54, 1.8459300258281917) (5.36, 2.336778806977256) (3.24, 1.469363421164984) (2.79, 2.4819123394508087) (1.77, 0.9080979040702231) (1.84, 0.4349265373865476) (2.88, 0.5531704058147365) (3.63, 2.1818040689263305) (1.88, 1.1244087432753498) (1.72, 1.8517602707347) (2.72, 1.73897873923067) (4.67, 1.5031547696597123) (3.14, 3.315798694048829) (4.78, 0.9993551031380891) (2.89, 3.204819628495066) (4.41, 3.1932872520641618) (1.92, 1.1093910939830125) (1.99, 1.0606500648504231) (4.88, 3.4429650400978327) (3.61, 1.5898835708273662) (4.36, 5.640581380589353) (5.41, 2.640298118535547) (2.88, 1.5394543912580985) (2.89, 1.185586208245325) (5.02, 2.5598745377201526) (4.08, 4.226679596558493) (1.82, 1.7627590985273969) (4.05, 1.8663783882512177) (5.14, 1.4661967706062207) (3.29, 4.241677831523453) (4.68, 2.8131112554849524) (3.71, 0.998391372359149) (4.31, 1.3083972488701283) (3.72, 2.525856780501219) (3.7, 0.17495255158612721) (3.55, 1.9553903054239257) (2.17, 1.5129644081340528) (1.84, 1.0238155476123554) (3.32, 1.7606649903548524) (7.49, 3.4738376384447087) (2.01, 1.7654893483932244) (5.02, 2.250276631162938) (5.11, 2.8323864393523523) (4.94, 6.725603585342234) (3.1, 1.5821787753394538) (3.43, 1.7354228823520457) (2.76, 1.0511565914891041) (4.95, 4.526919309826343) (2.01, 0.40603682784730133) (3.94, 2.057498683250021) (1.85, 0.7019136831488043) (6.34, 1.958947920451349) (6.14, 1.0317729300210936) (2.21, 4.68138529386664) (3.57, 2.00390637660083) (3.78, 1.9112195359138917) (3.28, 1.7945853624664732) (3.01, 1.839414493893701) (8.11, 4.405463186847249) (4.45, 1.9434430264617304) (3.76, 3.129712464788501) (2.8, 1.314514868477164) (4.75, 5.853382900902106) (1.82, 0.31807095910607247) (4.01, 1.4183372113400936) (1.78, 0.7854082728059248) (3.88, 2.0345602398054465) (2.82, 1.061245506078483) (1.85, 0.8958298977790918) (1.68, 0.5851787001357925) (1.72, 1.9188579094748106) (5.22, 12.317332549307972) (1.76, 0.8186536699397335) (4.68, 2.4014764744546673) (2.85, 2.3619117438999626) (3.23, 1.0971792470739021) (5.81, 1.0449897091998375) (1.83, 1.0538411120216589) (2.76, 1.8658796962017323) (3.1, 1.2931148008390407) (2.83, 2.2599944932820444) (1.85, 1.0649414390568086) (4.78, 1.958761830300334) (2.01, 0.584634158143667) (3.59, 2.4605714967444166) (3.46, 1.5204717569208601) (3.27, 2.8122486510261395) (3.41, 1.593189178886724) (4.53, 2.070934671742104) (3.49, 2.0010336391185857) (2.15, 1.757242921283165) (5.15, 3.086871500591922) (2.8, 2.197786364524088) (1.96, 1.7434661201207264) (2.15, 1.2054912349628915) (4.98, 5.066631581666449) (4.46, 2.722603584079433) (4.31, 1.7761219195143236) (5.0, 2.5176175601844553) (3.1, 2.6477951857921322) (2.75, 1.5806584144737277) (1.97, 0.907037522811418) (3.22, 2.388875201698875) (3.75, 1.6113575517447254) (3.94, 2.348191170317388) (1.67, 0.3968659740383568) (4.84, 1.8863570388897504) (4.04, 1.7899169559224597) (1.93, 1.6549958268297797) (1.81, 0.4782779022869175) (1.92, 1.523141621509809) (3.11, 1.5205668598048703) (2.58, 1.2583987862078965) (3.05, 2.380275653855313) (2.01, 1.9333542274475253) (3.08, 1.6043210702711752) (4.66, 3.305801479595844) (1.8, 2.19933863959474) (1.77, 1.0804561746909807) (3.62, 1.8418200151410533) (2.32, 1.729950863168833) 
};
\addlegendentry{genome-edit-distance}

\addplot[only marks,color=chocolate3,mark=+] coordinates {
(4.31, 1.3990774130208177) (4.59, 2.42909213199757) (4.75, 4.731900775745981) (11.98, 3.6042728151730983) 
};
\addlegendentry{labyrinth}

\addplot[only marks,color=chameleon3,mark=+] coordinates {
(357.04, 0.9482499073967605) (391.9, 0.995605115543755) (458.82, 0.975359484708951) (459.14, 0.8844601887108582) (342.08, 0.9924733785389105) (619.97, 1.0388875240283757) (532.65, 0.9370981469534772) (390.9, 1.0460040467829812) (441.62, 1.1239440424855327) (269.23, 1.0216653131101654) (445.32, 1.0260015990094677) (319.85, 0.9072804339343585) 
};
\addlegendentry{logistics-large}

\addplot[only marks,color=skyblue3,mark=+] coordinates {
(189.72, 4.826098181532395) (942.46, 4.941102724293341) (5420.15, 6.268158186148097) (9.33, 2.1384064394662814) (76.63, 4.128534445486763) 
};
\addlegendentry{organic-synthesis}

\addplot[only marks,color=plum3,mark=+] coordinates {
(3.93, 4.28394331808538) (6.98, 5.4536799257687285) (11.1, 1.8884359362045682) (7.12, 2.7294740391223784) (9.45, 1.8025232752146279) (7.19, 1.2678339997999872) (5.18, 3.7377854042820684) (13.65, 3.3189750903871014) (7.51, 3.534888178204117) (25.26, 3.531868959654933) (4.1, 1.7883322644332245) (3.0, 4.811077832038509) (7.92, 3.5702649772520823) (3.26, 6.156950389924489) (5.3, 3.7449406783222456) (4.05, 3.113061081288673) (4.97, 5.631863389699412) 
};
\addlegendentry{pipesworld-tankage-nosplit}

\addplot[only marks,color=scarletred3,mark=+] coordinates {
(59.14, 2.270942452067752) (65.7, 2.584889732703973) (45.38, 4.688890829316331) (270.94, 1.879437108087128) (156.77, 3.573306819538464) (124.26, 2.172598871151839) (161.94, 2.440205134398647) (152.41, 4.701807698566255) (85.14, 4.104367420469573) (261.1, 2.752804778252842) (28.15, 2.3752249725086085) (77.67, 2.9453836905709005) (70.03, 2.5513303442664057) (228.7, 2.6517021755797257) (251.92, 3.1691295675846236) (162.77, 2.3196988461481935) (91.37, 2.459044611741554) (157.89, 5.167259069852104) (207.36, 2.3624190066520594) (110.48, 2.3700592308632333) (116.07, 2.5719760077604485) (42.25, 4.176820466972517) (171.4, 2.3597715468341907) (91.72, 3.6955629490604567) (82.75, 5.601708776387027) (279.7, 2.4828674015607195) (97.11, 2.172286948364947) (210.29, 2.068452582146869) (40.6, 4.073934668730949) (183.26, 2.2807133850833474) (149.66, 2.302150990509508) (96.01, 2.5879182827541856) (151.18, 5.509986175375968) 
};
\addlegendentry{rovers-large-simple}

\addplot[only marks,color=aluminium3,mark=+] coordinates {
(1.02, 4.557407830860494) (1.03, 6.902076169203852) (1.03, 5.443992928771765) (1.04, 5.054509649001276) (1.04, 3.5815329775257605) (1.05, 5.134256988325328) (1.02, 3.2729138075301223) (1.01, 3.9597308728082266) (1.02, 3.563058368109108) (1.03, 5.665321903145862) (1.06, 3.3745715683827004) (1.02, 4.606947583276563) (1.03, 3.40437063501838) (1.04, 3.1279301971765454) (1.03, 5.32994505204196) (1.14, 3.297066564214165) (1.04, 3.0116707314084783) (1.06, 2.919713749624693) (1.03, 3.262946960038295) (1.08, 2.6993462184059154) (1.04, 5.12180107619836) (1.01, 4.7733766118765475) (1.02, 3.3729562809534106) (1.05, 3.5274874986227025) (1.07, 2.9590404228510754) (1.05, 3.6663103735220504) (1.04, 3.899328736506942) (1.04, 3.974360851880727) (1.04, 4.38892724394769) (1.02, 5.170735716732145) (1.11, 2.5337531604050714) (1.07, 3.4940979903745957) (1.01, 4.193672311751656) (1.03, 5.6177106775306624)
};
\addlegendentry{visitall-multidimensional}

\draw[color=black] (axis cs:1e-70,1) -- (axis cs:1e70,1);

\end{axis}
\end{tikzpicture}
\caption{Per domain average rule skew against relative speed from Tyr-1 to Tyr-8. The horizontal line depicts no speedup.}
\label{fig:skew-speedup}
\end{figure}
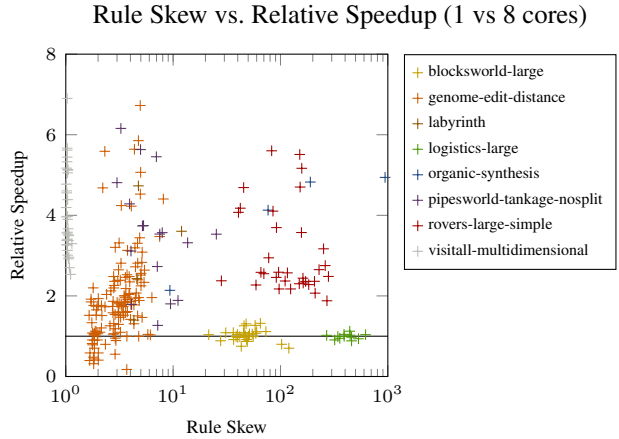

We now inspect the practical speedups observed from Tyr-1 to Tyr-8. Figure~\ref{fig:skew-speedup} compares, for a subset of tasks, the \defined{rule skew}, defined as the ratio of the maximum to the median total wall-clock rule execution time across all rules, and the \defined{speedup}, defined as the ratio of the total wall-clock execution time of Tyr-1 to that of Tyr-8. Most speedups fall in the range of 2 to 4, and fewer in the range of 4 to 7. For the domains in Table~\ref{tab:overview} where Tyr shows speedup improvement, i.e., Visitall-Multidim., Pipesworld-Tankage, Logistics-Large, and Rovers-Large, we make the following observations. In Visitall, rule skew is close to 1, and we observe speedups of at least 2 and up to 6. In the others, we still see significant speedups, but lower than in Visitall, motivating better load balancing via grounding parallelization. Similarly, runtime improvements occur in Rovers-Large, where Tyr-1 already solves all tasks. In Blocksworld-Large, where rule skew is consistently high, we observe no coverage improvements.

\subsection{Overall Results}

\newcommand{\numtasks}[1]{\small{(#1)}}
\setlength{\tabcolsep}{3pt}

\begin{table*}[ht]
  \centering
  \begin{tabular}{l l rr rr rr rr rr rr rr}
    \toprule
    & & \multicolumn{4}{c}{Baselines}
      & \multicolumn{10}{c}{Ours}                      \\
      \cmidrule(l){3-6}
      \cmidrule(lr){7-16}
    & Domain
      & \multicolumn{2}{c}{FD}
      & \multicolumn{2}{c}{PL}
      & \multicolumn{2}{c}{Tyr-1}
      & \multicolumn{2}{c}{Tyr-2}
      & \multicolumn{2}{c}{Tyr-4}
      & \multicolumn{2}{c}{Tyr-8}
      & \multicolumn{2}{c}{Tyr-8-2} \\
      \cmidrule(l){3-4}
      \cmidrule(l){5-6}
      \cmidrule(l){7-8}
      \cmidrule(l){9-10}
      \cmidrule(l){11-12}
      \cmidrule(lr){13-14}
      \cmidrule(l){15-16}
    & & S   & T
      & S   & T
      & S   & T
      & S   & T
      & S   & T
      & S   & T
      & S   & T \\

    \midrule

    & \textbf{ASA Summary \numtasks{990}}
      & \textbf{622} & 0.96
      & 487          & 4.73
      & 494          & 2.94
      & 506          & 2.25
      & 509          & 2.44
      & 503          & 2.99
      & 506          & 2.97           \\

    \midrule

    {\multirow{9}{*}{\rotatebox[origin=c]{90}{Hard-To-Ground}}}

    & Blockworld-Large \numtasks{40}
      & 12           & 9.41
      & 7            & 36.86
      & 39           & 0.73
      & \textbf{40}  & 0.75
      & 39           & 0.82
      & 39           & 0.80
      & \textbf{40}  & 0.77           \\
    & Childsnack-Large \numtasks{144}
      & 113          & 1.38
      & 78           & 4.82
      & \textbf{124} & 1.51
      & 123          & 1.52
      & 123          & 1.64
      & 122          & 1.62
      & 123          & 1.61           \\
    & Genome-Edit-Distance \numtasks{312}
      & \textbf{312} & 1.89
      & 306          & 5.07
      & 303          & 3.71
      & 308          & 2.86
      & 309          & 2.81
      & 311          & 3.57
      & 311          & 3.52           \\
    & Labyrinth \numtasks{40}
      & \textbf{12}  & --
      & 0            & --
      & 7            & --
      & 8            & --
      & 8            & --
      & 8            & --
      & 10           & --           \\
    & Logistics-Large \numtasks{40}
      & 36           & 207.31
      & \textbf{40}  & 1.89
      & \textbf{40}  & 3.95
      & \textbf{40}  & 3.80
      & \textbf{40}  & 3.87
      & \textbf{40}  & 3.82
      & \textbf{40}  & 3.82           \\
    & Organic-Synthesis \numtasks{56}
      & 21           & 3.78
      & \textbf{48}  & 0.04
      & 39           & 0.14
      & 39           & 0.14
      & 39           & 0.14
      & 39           & 0.17
      & 39           & 0.16           \\
    & Pipesworld-Tankage \numtasks{50}
      & 19           & 5.66
      & 26           & 0.90
      & 28           & 2.07
      & 33           & 1.58
      & \textbf{36}  & 1.29
      & \textbf{36}  & 1.23
      & 35           & 1.21           \\
    & Rovers-Large \numtasks{40}
      & 14           & 99.34
      & \textbf{40}  & 15.89
      & \textbf{40}  & 12.68
      & \textbf{40}  & 8.93
      & \textbf{40}  & 5.96
      & \textbf{40}  & 5.56
      & \textbf{40}  & 5.42           \\
    & Visitall-Multidim. \numtasks{180}
      & 66           & 16.22
      & 134          & 0.67
      & 131          & 0.46
      & 134          & 0.37
      & 139          & 0.31
      & 142          & 0.32
      & \textbf{143} & 0.29           \\
    \midrule

    & \textbf{HTG \numtasks{902}}
      & 605          & 10.72
      & 679          & 2.21
      & 751          & 1.17
      & 765          & 0.96
      & 772          & 0.90
      & 776          & 0.95
      & \textbf{781} & 0.95           \\

    \midrule
    & \textbf{Total Summary \numtasks{1892}}
      & 1227          & 1.64
      & 1166          & 3.50
      & 1245          & 2.56
      & 1271          & 1.94
      & 1281          & 1.97
      & 1279          & 2.30
      & \textbf{1287} & 2.29           \\

    \bottomrule
  \end{tabular}
  \caption{
      Comparison of different planners using lazy greedy best-first search with the FF heuristic. We compare the ground implementation in Fast Downward (FD), the lifted implementation in Powerlifted (PL), and our lifted implementation with N cores (Tyr-N). For each domain, we report the number of solved tasks (S), and, for commonly solved tasks, the geometric mean of the total time in seconds (T). For each domain, the lifted planners with the highest number of solved tasks are shown in bold.}
      \label{tab:overview}
\end{table*}

Table~\ref{tab:overview} shows that no single planner dominates uniformly across all domains. Instead, the results reveal a clear split between benchmark sets and planner types. On the Autoscale Agile benchmark set, FD achieves the highest coverage, indicating that grounding is often not the main bottleneck on these tasks and that the mature ground-search implementation remains highly competitive. On the Hard-To-Ground benchmark set, however, the lifted planners clearly outperform FD in coverage, which is consistent with the intended structure of these benchmarks: avoiding full grounding becomes advantageous when the grounded representation is difficult to construct or store.

Within the lifted planners, Tyr improves most consistently on HTG, where Tyr-8-2 solves 781 tasks compared to 679 for PL and 605 for FD. This advantage is not uniform across domains, suggesting that performance depends on the interaction between the lifted representation, the generated Datalog program, and the structure of the search space. The two lifted planners also make different design choices: PL applies transformations to the Datalog program to improve join-based evaluation, whereas Tyr avoids introducing additional relations in order to reduce synchronization overhead. For example, Tyr performs particularly well on Blocksworld-Large, whereas PL remains stronger on Organic-Synthesis. We therefore interpret these domain-level differences primarily as complementary strengths rather than as evidence of a uniform dominance relation.

The runtime results on HTG further indicate that Tyr benefits from rule-level parallelism on domains with sufficiently expensive Datalog evaluation. Runtime decreases almost consistently as the number of rule-level workers increases, for example, on Rovers-Large. However, the additional grounding-level parallelism in Tyr-8-2 yields only modest gains. This suggests that, for the current implementation and benchmark set, the dominant source of exploitable parallelism is coarse rule-level evaluation rather than splitting individual rules into multiple grounding tasks. This limitation is specific to the simple grounding-level splitting strategy evaluated here; it does not preclude stronger gains from finer-grained Datalog parallelization, such as the parallel relation operations used in systems like Soufflé~\cite{jordan-et-al-iccav2016}.

\section{Conclusions}

We presented a lifted planner based on semi-naive bottom-up Datalog evaluation with two levels of parallelism: rule-level parallelism and grounding parallelism. Our approach uses a $\ddelta$-edge-anchored clique-enumeration scheme to generate the ground rule instances required for semi-naive evaluation without duplicates. The empirical results indicate that this architecture scales well with the number of cores, typically achieving speedups of 2--4x and up to 6x on 8 cores, while our parallel-fraction analysis suggests a speedup limits of about 13x. In addition, the resulting planner substantially improves overall benchmark coverage, establishing a new state of the art for lifted FF-style planning in our evaluation. Future work includes better load-balancing strategies and combining our approach with hash-distributed search to scale the method horizontally and further improve performance.

\section{Acknowledgements}

This work was supported by the Swedish Foundation for Strategic Research and the Wallenberg AI, Autonomous Systems and Software Program (WASP) funded by the Knut and Alice Wallenberg Foundation. The computations were enabled by resources provided by the National Academic Infrastructure for Supercomputing in Sweden (NAISS), partially funded by the Swedish Research Council through grant agreement no.\ 2022-06725.

\bibliography{abbrv,literatur,crossref,extra}

\end{document}